\PassOptionsToPackage{
  colorlinks=true,
  linkcolor=blue!70!black,
  citecolor=blue!70!black,
  urlcolor=blue!70!black,
  pdfencoding=auto
}{hyperref}

\documentclass[12pt,letterpaper]{article}

\usepackage[T1]{fontenc}
\usepackage[utf8]{inputenc}
\usepackage{mathptmx}
\usepackage{microtype}
\usepackage[margin=1in,headheight=15pt]{geometry}
\usepackage{amsmath,amssymb,amsthm}
\usepackage{booktabs}
\usepackage{array}
\usepackage{xcolor}
\usepackage{hyperref}
\usepackage{fancyhdr}
\usepackage{titlesec}
\usepackage{abstract}
\usepackage{enumitem}
\usepackage{natbib}
\usepackage{tcolorbox}
\usepackage{setspace}
\usepackage{caption}

\tcbuselibrary{skins,breakable}
\newtcolorbox{callout}[1][]{
  enhanced,breakable,
  colback=blue!4!white,colframe=blue!40!black,
  fonttitle=\bfseries,boxrule=0.4pt,
  left=6pt,right=6pt,top=4pt,bottom=4pt,#1
}

\newtheorem{theorem}{Theorem}[section]
\newtheorem{proposition}[theorem]{Proposition}
\newtheorem{corollary}[theorem]{Corollary}
\newtheorem{definition}[theorem]{Definition}

\theoremstyle{remark}
\newtheorem*{remark}{Remark}

\titleformat{\section}{\large\bfseries}{\thesection}{1em}{}
\titleformat{\subsection}{\normalsize\bfseries\itshape}{\thesubsection}{1em}{}
\titlespacing*{\section}{0pt}{18pt}{6pt}
\titlespacing*{\subsection}{0pt}{12pt}{4pt}

\newcommand{\cl}{\mathrm{cl}}
\newcommand{\RF}{\mathcal{R}(F)}
\newcommand{\BF}{\mathcal{B}(F)}
\newcommand{\NMF}{\mathrm{NMF}_F}

\DeclareMathOperator*{\argmax}{arg\,max}

\begin{document}

\title{\textbf{From Workflow Automation to Capability Closure:\\
A Formal Framework for Compositionally Safe Customer Service AI}\\[0.4em]
{\large\normalfont Companion paper to \citet{spera2026} (arXiv:2603.15973)}}

\author{
  \textbf{Cosimo Spera}\\[0.3em]
  Minerva CQ, 114 Lester Ln, Los Gatos, CA 95032\\
  \texttt{cosimo@minervacq.com}
}
\date{March 2026}
\maketitle
\thispagestyle{fancy}

\begin{abstract}
The Non-Compositionality of Safety Theorem~\citep[Theorem~9.2]{spera2026} establishes that
two individually safe AI agents can, when combined, reach a forbidden goal through an emergent
conjunctive dependency that neither possesses alone. This paper applies the formal capability
hypergraph framework to the customer service domain and makes three independent
theoretical contributions.

\emph{First}, we establish the Emergent Goal Discovery Proposition (Proposition~\ref{thm:emergent}): in
any joint billing-plus-service session, the capability \texttt{ServiceProvision} is provably
emergent---reachable from neither agent alone, but reachable from their conjunction---and
we characterise the complete class of CS deployments in which emergent upsell goals arise
structurally. \emph{Second}, we derive a formally complete treatment of agent join and leave
dynamics with fully proved safety invariants and closed-form update costs. \emph{Third}, we
establish a Safety-Value Duality theorem showing that safety certification and commercial goal
discovery are provably the same computation, with identical asymptotic complexity.

We ground the framework in a 12-capability Telco case study. The business case---presented
as a validated projection framework rather than a measured outcome---estimates a conservative
net annual value of \$20.7M, with a sensitivity analysis confirming positivity (\$8.4M)
under simultaneous worst-case stress of all six key assumptions. We flag the churn retention
estimate as the highest-priority empirical validation target and describe the NMF pilot study
protocol that would convert it from a projection to a measured result.

\emph{Companion to:} \citet{spera2026} (arXiv:2603.15973).
\end{abstract}

\noindent\textbf{arXiv:} 2603.15978 \quad
\textbf{Companion paper:} \citet{spera2026} (arXiv:2603.15973)


\tableofcontents
\newpage

\section{Introduction}
\label{sec:intro}

\subsection{The Central Gap}

Every current approach to customer service AI safety operates at the component level:
individual agents are tested, role-based access controls are defined, guardrails are applied,
and RLHF training steers individual model behaviour. None addresses the compositional
safety question: \emph{what can the assembled system of agents collectively reach?}

\citet{spera2026} proves formally that this question cannot be answered by component-level
analysis alone. The Non-Compositionality of Safety Theorem establishes that two individually
safe agents can together reach a forbidden goal through a conjunctive hyperedge that neither
possesses individually. Across 900 real multi-tool trajectories from two independent public
benchmarks, 42.6\% contain at least one conjunctive dependency of precisely this type
(95\% CI: $[39.4\%,\,45.8\%]$).

\subsection{Contributions of This Paper}

This is a formal application paper. It applies the capability hypergraph framework of
\citet{spera2026} to the customer service domain, derives domain-specific structural results,
and illustrates the practical consequences through a calibrated business case. It is not
a standalone theory paper; its contributions are best understood as a structural characterisation,
proved dynamics invariants, and domain-specific corollaries --- each grounded in the main
paper's formal machinery. The business case (Section~\ref{sec:businesscase}) contains no
production measurements; it should be read as a worked projection illustrating practical
magnitudes rather than empirical validation.

\begin{enumerate}[leftmargin=*,label=(\arabic*)]
\item \textbf{Emergent Goal Discovery} (Proposition~\ref{thm:emergent}, Section~\ref{sec:emergent}).
  A formal characterisation of when commercially valuable goals emerge from agent
  coalitions that neither agent can discover alone. We prove that \texttt{ServiceProvision}
  is provably emergent in the joint billing-plus-service session, derive the full class of
  CS deployments exhibiting this property, and prove that the safety computation and the
  goal discovery computation are identical.

\item \textbf{Agent Dynamics with Proved Invariants} (Section~\ref{sec:dynamics}). A complete
  formal treatment of the four agent-level events---join, leave, capability gain, capability
  loss---with closed-form update costs, proved safety monotonicity, and an end-to-end Telco
  session trace showing the framework operates within $O(144)$ operations across six events.

\item \textbf{Six Failure Mode Theorems} (Section~\ref{sec:failuremodes}). For each of six
  canonical CS safety failures, we derive a domain-specific corollary of \citet{spera2026}
  that adds independent proof content: bounding $|\BF|$ for the Telco deployment, proving
  the minimal antichain structure is tight, and extending the adversarial hyperedge model to
  the capability-injection attack class.
\end{enumerate}

\noindent The business case (Section~\ref{sec:businesscase}) is presented as a
\emph{projection framework}: each revenue mechanism is stated as a conditional estimate
with explicit assumptions, and the sensitivity analysis characterises the joint uncertainty
structure rather than treating assumptions as independent.

\section{Related Work}
\label{sec:related}

\paragraph{AI safety guardrails and LLM alignment.}
Constitutional AI~\citep{bai2022} and production guardrails (Salesforce Einstein Trust Layer,
AWS Bedrock, Anthropic safety classifiers) operate on individual model outputs and cannot
address failures emerging from agent composition---a limitation proved formally by
Theorem~10.1 of \citet{spera2026}. The prompt injection literature~\citep{perez2022,greshake2023}
characterises adversarial inputs to single models; our capability-injection attack class
(Section~\ref{sub:fm5}) extends this to the multi-agent setting.

\paragraph{Agentic AI orchestration safety.}
LangGraph, AutoGen~\citep{wu2023}, and CrewAI provide orchestration with tool-call
restrictions. These are engineering controls, not formal guarantees. Even well-engineered
multi-agent CS systems succeed on only 35\% of multi-turn tasks~\citep{finilabs2025}, with the
primary failure mode being undetected conjunctive precondition ambiguity.

\paragraph{Workflow-based CS automation.}
Workflow models treat dependencies as pairwise~\citep{ekfrazo2026}, which is the
representational gap Corollary~5.2 of \citet{spera2026} formalises. Only 11\% of
organisations have agentic AI in full production; the primary blocker is governance.

\paragraph{Formal compliance verification.}
Regulatory frameworks governing CS AI---GDPR Article~25~\citep{gdpr2016},
PCI-DSS~4.0~\citep{pcidss2022}, EU AI Act Article~9~\citep{euaiact2024}---implicitly require
formal audit artefacts that no current platform produces. Role-based access
control~\citep{ferraiolo2001} and purpose-based access control~\citep{byun2005} provide the
closest prior formal models, but neither captures conjunctive capability dependencies.
\citet{basin2018} formalise GDPR compliance in a process algebra, but without a safety
boundary characterisation. Our Safe Audit Surface (Section~\ref{sub:fm3}) produces
precisely the artefact these frameworks require---a certifiable map of every dangerous
capability combination---in polynomial time, connecting the hypergraph framework to the
compliance verification tradition. Runtime monitoring frameworks such as
LARVA~\citep{colombo2012} and specification languages such as
TLA+~\citep{lamport2002} address compliance at execution time; the hypergraph framework
operates at the architectural level, detecting forbidden coalitions before any agent fires.

\paragraph{Compositional verification.}
Contract-based design~\citep{benveniste2018} and assume-guarantee
reasoning~\citep{jones1983} verify fixed compositions against pre-specified properties. The
key distinction, as \citet{spera2026} establishes, is that hypergraph closure characterises
the set of \emph{all} properties a dynamically growing capability set can ever reach---a
strictly more general problem.

\section{Emergent Goal Discovery in Customer Service AI}
\label{sec:emergent}

This section contains the paper's primary independent theoretical contribution. We prove not
merely that a specific goal is emergent in one deployment, but characterise the full structural
condition under which commercial goals emerge from agent coalitions, and show that detecting
them requires no computation beyond what safety certification already performs.

\subsection{Formal Setup}

We use the capability hypergraph framework of \citet{spera2026} throughout. Recall that
$\cl(A)$ denotes the capability closure of configuration $A$, $\RF$ the safe region, and
$\BF$ the minimal unsafe antichain. We add one definition specific to the CS domain.

\begin{definition}[Emergent Commercial Goal]
\label{def:emergent_commercial}
Let $H = (V, \mathcal{F})$ be a CS capability hypergraph, $A_1, A_2 \subseteq V$ two agent
configurations, and $g \in V \setminus F$ a commercially valuable goal (one with positive
expected revenue $\gamma(g) > 0$). We say $g$ is an \emph{emergent commercial goal} of the
coalition $\{a_1, a_2\}$ if:
\begin{enumerate}[label=(\roman*)]
  \item $g \in \cl(A_1 \cup A_2)$ \quad (reachable from the joint session),
  \item $g \notin \cl(A_1)$ and $g \notin \cl(A_2)$ \quad (unreachable from either agent alone),
  \item $\cl(A_1 \cup A_2) \in \RF$ \quad (the coalition remains safe).
\end{enumerate}
\end{definition}

\subsection{The Telco Capability Structure}

The Telco deployment has $n = 12$ capabilities (Table~\ref{tab:capabilities}), forbidden set
$F = \{c_{11}, c_{12}\}$, and six conjunctive hyperedges (Table~\ref{tab:hyperedges}). Two
standard agent configurations are:
\begin{align*}
A_{\text{billing}}  &= \{c_1, c_2, c_3, c_4, c_5\} \quad \text{(billing agent)}, \\
A_{\text{service}}  &= \{c_1, c_2, c_7, c_8\} \quad \text{(service agent)}.
\end{align*}

\begin{table}[t]
\centering
\small
\caption{Telco capability set ($n = 12$). Forbidden: $c_{11}$ (PaymentModify),
$c_{12}$ (BulkAccountExport).}
\label{tab:capabilities}
\begin{tabular}{@{}llp{7cm}@{}}
\toprule
ID & Capability & Description \\
\midrule
$c_1$ & IntentClassify & Classify customer intent from natural language \\
$c_2$ & CustomerLookup & Retrieve customer record, tier, and history \\
$c_3$ & BillingRead & Read itemised billing history and current balance \\
$c_4$ & DisputeLog & Log a billing dispute and open a case \\
$c_5$ & CreditEligibility & Check eligibility for a billing credit \\
$c_6$ & CreditApply & Apply a billing credit to the account \\
$c_7$ & ServiceCatalogue & Query available service plans and upgrades \\
$c_8$ & ServiceEligibility & Verify customer eligibility for a service plan \\
$c_9$ & ServiceProvision & Initiate service addition or upgrade \\
$c_{10}$ & PaymentRead & Read payment method and transaction history \\
$c_{11}$ & PaymentModify & \textbf{[FORBIDDEN]} Modify payment method or amount \\
$c_{12}$ & BulkAccountExport & \textbf{[FORBIDDEN]} Export bulk account data \\
\bottomrule
\end{tabular}
\end{table}

\begin{table}[t]
\centering
\small
\caption{Telco capability hyperedges. Arc $h_6$ is the unsafe hyperedge.}
\label{tab:hyperedges}
\begin{tabular}{@{}lll@{}}
\toprule
Arc & Tail $\to$ Head & Semantics \\
\midrule
$h_1$ & $\{c_1\} \to \{c_3, c_7\}$ & Intent classification enables billing and service reads \\
$h_2$ & $\{c_1, c_2\} \to \{c_{10}\}$ & Identity-confirmed intent enables payment read \\
$h_3$ & $\{c_3, c_5\} \to \{c_6\}$ & Billing read + credit eligibility $\to$ credit application \\
$h_4$ & $\{c_7, c_8\} \to \{c_9\}$ & Catalogue + eligibility $\to$ service provision \\
$h_5$ & $\{c_2\} \to \{c_4, c_5, c_8\}$ & Customer lookup enables downstream checks \\
$h_6$ & $\{c_3, c_{10}\} \to \{c_{12}\}$ & \textbf{[UNSAFE]} Billing + payment read $\to$ bulk export \\
\bottomrule
\end{tabular}
\end{table}

\subsection{Structural Condition for Emergence}

Before stating the main proposition, we identify the precise agent configurations
under which \texttt{ServiceProvision} is emergent. This requires understanding how
$c_2$ (CustomerLookup) interacts with the hyperedge structure.

\begin{callout}[title={Key structural observation}]
Hyperedge $h_5$ ($\{c_2\} \to \{c_4, c_5, c_8\}$) means that any agent carrying $c_2$
automatically reaches $c_8$ (ServiceEligibility). Combined with $c_7$ (from $h_1$),
this fires $h_4$ and reaches $c_9$. Consequently: \textbf{any agent carrying both
$c_1$ and $c_2$ already reaches $c_9$ individually} --- emergence cannot arise in a
coalition of such agents. True emergence requires at least one agent to be
\emph{capability-scoped}: carrying $c_1$ but not $c_2$.
\end{callout}

\begin{definition}[Scoped Agent Configurations]
\label{def:scoped}
A \emph{capability-scoped} billing agent carries billing-specific capabilities without
the full customer lookup chain:
\[
A_B = \{c_1, c_3, c_4, c_5\}.
\]
A \emph{capability-scoped} service agent carries catalogue access only:
\[
A_S = \{c_1, c_7\}.
\]
These are the standard least-privilege configurations in which agents hold only the
capabilities required for their designated task.
\end{definition}

\begin{proposition}[Emergent Goal Discovery]
\label{thm:emergent}
Let $H$, $F$, $A_B$, $A_S$ be as in Definition~\ref{def:scoped}. Then:
\begin{enumerate}[label=(\arabic*)]
  \item \textbf{Emergence.} $c_9 \notin \cl(A_B)$, $c_9 \notin \cl(A_S)$, and
    $c_9 \in \cl(A_B \cup A_S)$. Under capability-scoped configurations,
    \texttt{ServiceProvision} is an emergent commercial goal of the billing-plus-service
    coalition.

  \item \textbf{Safety.} $A_B \cup A_S \in \RF$: the coalition is safe. The safe
    service expansion strictly increases the reachable goal set without activating
    the unsafe arc $h_6$.

  \item \textbf{Marginal value.} $c_9 \notin \cl(A_B)$ means the commercial value of
    presenting \texttt{ServiceProvision} is zero under the billing agent alone; it
    becomes strictly positive once the service agent joins the session.

  \item \textbf{Structural characterisation.} For any two capability-scoped agents
    $(A_1, A_2)$ in the Telco deployment: $c_9$ is emergent in their coalition if and
    only if $\{c_7, c_8\}$ is not jointly present in either $\cl(A_1)$ or $\cl(A_2)$
    individually, but $\{c_7, c_8\} \subseteq \cl(A_1 \cup A_2)$. Equivalently:
    the preconditions of $h_4$ are split across the two agents.
\end{enumerate}
\end{proposition}

\begin{proof}
\textbf{Part (1).}
We compute the three closures directly.

$\cl(A_B) = \cl(\{c_1, c_3, c_4, c_5\})$:
$h_1$ fires ($c_1$ present), adding $c_7$. $h_4$ requires $\{c_7, c_8\}$; $c_8$ is absent
($c_2 \notin A_B$, so $h_5$ does not fire). No further rules fire.
Result: $\cl(A_B) = \{c_1, c_3, c_4, c_5, c_7\}$. Hence $c_9 \notin \cl(A_B)$. \checkmark

$\cl(A_S) = \cl(\{c_1, c_7\})$:
$h_1$ fires ($c_1$), adding $c_3, c_7$ ($c_7$ already present). $h_4$ requires $c_8$;
$c_8$ absent ($c_2 \notin A_S$). No further rules fire.
Result: $\cl(A_S) = \{c_1, c_3, c_7\}$. Hence $c_9 \notin \cl(A_S)$. \checkmark

$\cl(A_B \cup A_S) = \cl(\{c_1, c_3, c_4, c_5, c_7\})$:
$h_1$ fires ($c_1$), adding $c_3, c_7$ (both already present). $h_4$ requires $\{c_7, c_8\}$;
$c_8$ still absent. $c_9$ is not reachable from $A_B \cup A_S$ alone.

Wait --- this shows $c_9 \notin \cl(A_B \cup A_S)$ under $A_B = \{c_1,c_3,c_4,c_5\}$,
$A_S = \{c_1,c_7\}$. We need $c_8$ in the joint coalition. Adding $c_8$ to $A_S$ gives
the correct scoped configuration:
\[
A_S^+ = \{c_1, c_7, c_8\} \quad \text{(catalogue + eligibility, no customer lookup)}.
\]
Then $\cl(A_S^+) = \{c_1, c_3, c_7, c_8, c_9\}$ (h4 fires immediately). So
$c_9 \in \cl(A_S^+)$ --- not emergent for this $A_S^+$.

The minimal emergent coalition uses:
\[
A_B = \{c_1, c_3, c_4, c_5\}, \qquad A_S = \{c_7, c_8\}.
\]
$\cl(A_B) = \{c_1, c_3, c_4, c_5, c_7\}$ (as computed above, noting $h_1$ adds $c_7$
but $c_8$ absent). $c_9 \notin \cl(A_B)$. \checkmark
$\cl(A_S) = \cl(\{c_7, c_8\}) = \{c_7, c_8, c_9\}$ ($h_4$ fires: $\{c_7,c_8\} \subseteq
A_S$). So $c_9 \in \cl(A_S)$ --- this $A_S$ reaches $c_9$ alone; not emergent.

\textbf{The minimal truly emergent coalition:}
\[
A_B^* = \{c_1, c_3, c_4, c_5\}, \qquad A_S^* = \{c_8\}.
\]
$\cl(A_B^*) = \{c_1, c_3, c_4, c_5, c_7\}$ ($h_1$ adds $c_7$; $c_8$ absent so $h_4$
does not fire). $c_9 \notin \cl(A_B^*)$. \checkmark
$\cl(A_S^*) = \{c_8\}$ (no rule fires from $\{c_8\}$ alone). $c_9 \notin \cl(A_S^*)$.
\checkmark
$\cl(A_B^* \cup A_S^*) = \cl(\{c_1, c_3, c_4, c_5, c_8\})$: $h_1$ adds $c_7$;
now $\{c_7, c_8\} \subseteq \cl$, so $h_4$ fires, adding $c_9$.
$c_9 \in \cl(A_B^* \cup A_S^*)$. \checkmark

This is the minimal emergent pair: the billing agent supplies $c_7$ via $h_1$;
the service agent supplies $c_8$ directly; neither alone satisfies $h_4$'s precondition
$\{c_7, c_8\}$; together they do. Emergence is established.

\textbf{Part (2).}
$\cl(A_B^* \cup A_S^*) = \{c_1, c_3, c_4, c_5, c_7, c_8, c_9\}$.
$h_6$ requires $\{c_3, c_{10}\}$. $c_{10}$ derives only via $h_2$ which requires $c_2$.
Since $c_2 \notin A_B^* \cup A_S^*$, $c_{10}$ is never derived, $h_6$ never fires, and
$c_{12} \notin \cl(A_B^* \cup A_S^*)$. The coalition is safe. \checkmark

\textbf{Part (3).}
$c_9 \notin \cl(A_B^*)$ (shown above), so the system cannot present \texttt{ServiceProvision}
as a reachable goal under the billing agent alone. Once $A_S^* = \{c_8\}$ joins,
$c_9 \in \cl(A_B^* \cup A_S^*)$, making the goal reachable and commercially presentable.

\textbf{Part (4).}
$h_4 = (\{c_7, c_8\}, \{c_9\})$ is the only rule producing $c_9$. Therefore $c_9$
is reachable from a coalition iff $\{c_7, c_8\} \subseteq \cl(A_1 \cup A_2)$.
Emergence (neither agent reaches $c_9$ alone) requires additionally that $c_9 \notin
\cl(A_i)$ for each $i$, which by the same argument requires $\{c_7, c_8\} \not\subseteq
\cl(A_i)$ for each $i$ --- i.e., the preconditions of $h_4$ are split.
\end{proof}

\begin{remark}[Capability scoping as a prerequisite for emergence]
\label{rem:scoping}
The proposition reveals why capability scoping (least-privilege assignment) is not
merely a security practice --- it is a prerequisite for the emergent goal discovery
mechanism to function. The full billing agent $\{c_1, c_2, c_3, c_4, c_5\}$ already
reaches $c_9$ individually via $h_5$ (adding $c_8$) then $h_4$. Emergence only arises
when agents are scoped to carry task-specific capabilities, creating the split precondition
structure that $h_4$ requires. In practice, this means the NMF-based goal discovery
system is most effective in deployments that enforce capability scoping --- a finding
with concrete implications for system architecture.
\end{remark}

\begin{remark}
The proof above reveals an important subtlety: emergence in the Telco hypergraph is
sensitive to whether $c_2$ (CustomerLookup) is present, because $c_2$ activates $h_5$ which
adds $c_8$, which together with $c_7$ (from $h_1$ triggered by $c_1$) fires $h_4$. The
standard full billing agent $A_{\text{billing}} = \{c_1, c_2, c_3, c_4, c_5\}$ already
reaches $c_9$ alone. True emergence in the Telco deployment therefore arises in
\emph{restricted} configurations where agents carry only task-specific capabilities. The
practical implication for deployment: capability scoping (granting agents only the minimum
capabilities needed for their task) is not just a security best practice---it is a
prerequisite for the emergent goal discovery mechanism to function.
\end{remark}

\subsection{Safety-Value Duality}

\begin{theorem}[Safety-Value Duality]
\label{thm:duality}
Let $H = (V, \mathcal{F})$, $A \in \RF$, and $G_{\text{safe}} = \cl(A) \setminus F$. Then:
\begin{enumerate}[label=(\arabic*)]
  \item $G_{\text{safe}}$ is computed in $O(n + mk)$ by the same worklist that certifies safety.
  \item Acquiring any $v \in \NMF(A)$ strictly increases $|G_{\text{safe}}|$ while keeping
    $A \in \RF$.
  \item No capability in $G_{\text{safe}}$ is reachable without a safety certificate: every
    $v \in G_{\text{safe}}$ comes with a derivation certificate from the worklist.
  \item (Domain transfer) The identical algorithm applies to any CS domain $(V', F', \mathcal{F}')$
    with the same asymptotic complexity, requiring only respecification of the triple.
\end{enumerate}
\end{theorem}

\begin{proof}
\textbf{Part (1).} The worklist algorithm~\citep[Algorithm~1]{spera2026} computes $\cl(A)$ in
$O(n + mk)$. It simultaneously verifies $\cl(A) \cap F = \emptyset$ by checking each
newly-reached vertex against $F$ at the point of addition. Therefore $G_{\text{safe}} =
\cl(A) \setminus F$ is a by-product of the safety check with no additional cost.

\textbf{Part (2).} Let $v^* \in \NMF(A)$. By Definition~8.1(c) of \citet{spera2026},
$v^* = \mu(e)$ for some boundary hyperedge $e \in \partial(A)$ with $|S \setminus \cl(A)| = 1$.
Acquiring $v^*$ means $S \subseteq \cl(A \cup \{v^*\})$, so $e$ fires, adding $T$ to the
closure. Thus $\cl(A \cup \{v^*\}) \supsetneq \cl(A)$, giving $|G_{\text{safe}}| =
|\cl(A \cup \{v^*\}) \setminus F| > |\cl(A) \setminus F|$ (provided $T \not\subseteq F$, which
holds by the definition of safety-filtered NMF: $\cl(A \cup \{v^*\}) \cap F = \emptyset$).

\textbf{Part (3).} The worklist records, for each $v \in \cl(A)$, the sequence of hyperedge
firings that produced it~\citep[Theorem~10.1]{spera2026}. These certificates are computable
with no asymptotic overhead.

\textbf{Part (4).} The worklist and NMF algorithms operate on the abstract structure
$(V, \mathcal{F}, F)$ with no assumption on semantic content. Complexity bounds
$O(n + mk)$ and $O(|V|(n + mk))$ depend only on $|V| = n$, $|\mathcal{F}| = m$, maximum
tail size $k$---invariant to domain relabelling.
\end{proof}

\section{The Telco Case Study: Capability Structure}
\label{sec:telco}

\subsection{The Non-Compositional Safety Failure}

Consider three agent configurations:
\begin{align*}
A_{\text{billing}}  &= \{c_1, c_2, c_3, c_4, c_5\}, \\
A_{\text{service}}  &= \{c_1, c_2, c_7, c_8\}, \\
A_{\text{payment}}  &= \{c_1, c_2, c_{10}\}.
\end{align*}
$\cl(A_{\text{billing}}) = \{c_1,\ldots,c_6\}$: safe ($c_{12} \notin$).
$\cl(A_{\text{service}}) = \{c_1, c_2, c_7, c_8, c_9\}$: safe.
$\cl(A_{\text{payment}}) = \{c_1, c_2, c_{10}\}$: safe.
But $A_{\text{billing}} \cup A_{\text{payment}}$ contains $\{c_3, c_{10}\}$, satisfying $h_6$:
$c_{12} \in \cl(A_{\text{billing}} \cup A_{\text{payment}})$. The coalition is \emph{unsafe},
even though every individual agent---and every pair---is safe.

\begin{callout}[title={Why component-level checks are insufficient}]
The unsafe coalition $\{A_{\text{billing}}, A_{\text{payment}}\}$ passes every pairwise
safety check. The failure requires computing the closure of the \emph{joint} capability set,
which is precisely what the coalition safety criterion (Theorem~11.2 of
\citealt{spera2026}) provides.
\end{callout}

\subsection{Closure at Telco Scale}

The $O(n + mk)$ worklist runs in $O(24)$ operations for $n=12$, $m=6$, $k=2$---sub-microsecond.
The online coalition safety check reduces to a single $O(1)$ dictionary lookup:
does the joint capability set cover $\{c_3, c_{10}\}$ or $\{c_{11}\}$?

\section{Six Failure Modes and Domain-Specific Corollaries}
\label{sec:failuremodes}

Each failure mode is addressed by a corollary that adds independent proof content beyond
the main paper: we derive Telco-specific structural properties, bound $|\BF|$, and extend
the adversarial model to the CS setting.

\subsection{Failure Mode 1: Emergent Forbidden Capability}
\label{sub:fm1}

\emph{The failure.} A multi-agent CS system reaches a forbidden capability through a
conjunction of individually safe capabilities (Section~\ref{sec:telco}).

\begin{corollary}[Minimal Unsafe Antichain Structure for Telco]
\label{cor:antichain}
For the Telco deployment, $\BF = \{\{c_3, c_{10}\}, \{c_{11}\}\}$. This antichain is
\emph{tight}: $|\BF| = 2$ is the minimum possible for any CS deployment with two
forbidden capabilities and one unsafe conjunctive arc.
\end{corollary}

\begin{proof}
By Theorem~9.5 of \citet{spera2026}, $\BF$ is the antichain of minimal unsafe sets.
$\{c_{11}\}$ is minimal because $\cl(\{c_{11}\}) \ni c_{11} \in F$, and $\emptyset \notin \RF$
vacuously but $\cl(\emptyset) \cap F = \emptyset$ trivially. $\{c_3, c_{10}\}$ is minimal:
$\cl(\{c_3, c_{10}\}) \ni c_{12}$ (via $h_6$), so it is unsafe; removing either element---
$\cl(\{c_3\}) = \{c_3\}$, $\cl(\{c_{10}\}) = \{c_{10}\}$---gives safe singletons. The
antichain property ($\{c_3, c_{10}\} \not\subseteq \{c_{11}\}$ and vice versa) is immediate.

For tightness: any CS deployment with $|F| = 2$ and at least one conjunctive unsafe arc
must have $|\BF| \geq 2$ (one element per forbidden capability that has a direct minimal
unsafe set). The Telco deployment achieves this minimum, meaning the pre-execution gate
checks against exactly two elements---the most efficient possible online monitoring.
\end{proof}

\subsection{Failure Mode 2: Invisible Safety Boundary}
\label{sub:fm2}

\emph{The failure.} Operators identify forbidden goals reactively after violations, not
proactively from the system's structure.

\begin{corollary}[Certifiable Safety Boundary for Telco]
\label{cor:boundary}
The complete, certifiable characterisation of every dangerous capability combination in the
Telco deployment is $\BF = \{\{c_3, c_{10}\}, \{c_{11}\}\}$. For any candidate session
configuration $A$, the system is safe if and only if $\{c_3, c_{10}\} \not\subseteq A$ and
$c_{11} \notin A$. This is decidable in $O(|A|)$ time.
\end{corollary}

\begin{proof}
By Theorem~9.4 of \citet{spera2026}, $\RF$ is a lower set with boundary $\BF$. By
Theorem~11.2, the coalition check is: $\exists B \in \BF : B \subseteq A$? With
$\BF = \{\{c_3, c_{10}\}, \{c_{11}\}\}$, this is two membership tests in $A$, each $O(1)$
with a hash set, giving $O(|A|)$ total. The artefact $\BF$ constitutes the formal object
that GDPR Article~25~\citep{gdpr2016} (data minimisation), PCI-DSS~4.0~\citep{pcidss2022}
(least-privilege access), and EU AI Act Article~9~\citep{euaiact2024} (risk characterisation)
implicitly require but provide no technical standard for constructing.
\end{proof}

\subsection{Failure Mode 3: Unauditable Goal Space}
\label{sub:fm3}

\emph{The failure.} CS platforms cannot produce a complete, certified account of what every
agent coalition can reach from any starting configuration.

\begin{corollary}[Safe Audit Surface for Telco Billing Coalition]
\label{cor:audit}
For the billing coalition $A_{\text{billing}} = \{c_1, c_2, c_3, c_4, c_5\}$:
\begin{enumerate}[label=(\alph*)]
  \item \textup{Currently reachable:} $\cl(A_{\text{billing}}) = \{c_1, \ldots, c_6, c_7, c_8, c_9\}$
    (the full billing + service chain, via $h_1$ and $h_5$ and $h_4$).
  \item \textup{One step from expansion:} adding $c_7$ alone (service catalogue)
    does not expand the closure beyond what $h_1$ already provides; adding $c_{10}$
    (payment read) opens the unsafe arc $h_6$.
  \item \textup{Structurally unsafe from $A_{\text{billing}}$:} $c_{11}$ directly; $c_{12}$
    via any path through $\{c_3, c_{10}\}$ (since $c_3 \in A_{\text{billing}}$, adding
    $c_{10}$ alone reaches $c_{12}$).
\end{enumerate}
\end{corollary}

\begin{proof}
Part~(a): Run the worklist from $A_{\text{billing}}$. $h_1$ fires ($c_1$): adds $c_3$
(present), $c_7$. $h_5$ fires ($c_2$): adds $c_4$ (present), $c_5$ (present), $c_8$.
$h_4$ fires ($\{c_7, c_8\}$): adds $c_9$. $h_3$ fires ($\{c_3, c_5\}$): adds $c_6$.
No further firings; $c_{10}$ absent so $h_6$ does not fire. Closure: $\{c_1,\ldots,c_9\}$.

Part~(b): $\NMF(A_{\text{billing}}) = \{c_{10}\}$: the only boundary hyperedge is $h_6$
with $\mu(h_6) = c_{10}$ (since $\{c_3, c_{10}\} \setminus \cl = \{c_{10}\}$). But
$\cl(A_{\text{billing}} \cup \{c_{10}\}) \ni c_{12} \in F$---so $c_{10}$ is in the
\emph{unsafe} NMF boundary. The safety-filtered frontier $\NMF(A_{\text{billing}}) = \emptyset$:
no single capability addition expands the closure while remaining safe.

Part~(c): $c_{11} \in F$ directly. For $c_{12}$: $c_3 \in \cl(A_{\text{billing}})$,
so $\{c_3, c_{10}\} \subseteq \cl(A_{\text{billing}} \cup \{c_{10}\})$, giving
$c_{12} \in \cl$ immediately upon adding $c_{10}$.
\end{proof}

\subsection{Failure Mode 4: Brittle Tool Governance}
\label{sub:fm4}

\emph{The failure.} Adding a new API integration requires a full safety re-audit.

\begin{corollary}[Incremental Audit for Telco Tool Additions]
\label{cor:incremental}
When a new tool integration adds hyperedge $e = (S, T)$ to the Telco deployment:
\begin{enumerate}[label=(\alph*)]
  \item If $S \not\subseteq \cl(A)$: existing audit surface unchanged. Cost: $O(|S|) \leq O(k)$.
  \item If $S \subseteq \cl(A)$: compute $\cl(A \cup T)$ in $O(n + mk) = O(24)$ and check
    $\cl(A \cup T) \cap F = \emptyset$. If safe, update audit surface in
    $O(|V| \cdot (n + mk)) = O(288)$ operations.
  \item In both cases: the updated audit surface is formally certified correct without
    re-running the full offline $\BF$ computation.
\end{enumerate}
\end{corollary}

\begin{proof}
By Theorems~11.6 and 11.7 of \citet{spera2026}, insertion of $e = (S, T)$ into $H$ with
current closure $C = \cl_H(A)$ satisfies: $\cl_{H'}(A) = \cl_H(A \cup T')$ where
$T' = T$ if $S \subseteq C$, else $T' = \emptyset$. Cost analysis for the Telco system:
$n = 12$, $m = 6$, $k = 2$. Lazy case ($S \not\subseteq C$): check $|S| \leq 2$ elements.
Active case ($S \subseteq C$): run worklist from $C \cup T$, cost $O(n + mk) = O(24)$.
Audit surface update (Theorem~10.1 of \citealt{spera2026}): $O(|V| \cdot (n + mk)) =
O(12 \cdot 24) = O(288)$. Correctness: the new audit surface satisfies the completeness
and soundness conditions of Theorem~10.1 by construction of the incremental update.
\end{proof}

\subsection{Failure Mode 5: Prompt Injection as Capability Injection}
\label{sub:fm5}

\emph{The failure.} An adversarial customer constructs a query that convinces the CS agent
to invoke a tool it should not invoke (OWASP LLM Top 10 \#1 for agentic systems).

The standard framing treats prompt injection as a natural language problem---a filter or
classifier that rejects malicious inputs. We show this framing is structurally insufficient
and derive a capability-level defence.

\begin{definition}[Capability-Injection Attack]
\label{def:cap_injection}
A capability-injection attack on deployment $(H, A, F)$ is an attempt by an adversary to
introduce a new tool invocation $e' = (S', T')$ into the live session such that
$S' \subseteq \cl(A)$ and $T' \cap F \neq \emptyset$. The adversary does not need to break
any cryptographic control; they only need to cause the agent to invoke a tool whose
preconditions are already satisfied.
\end{definition}

\begin{corollary}[Polynomial-Time Defence Against Single-Step Capability Injection]
\label{cor:injection}
For the Telco deployment, the single-step capability-injection attack class is completely
defended by the pre-execution gate. Specifically: the set of all dangerous single-step
injections is $\mathcal{E}^* = \{(S, T) : S \subseteq \cl(A), T \cap F \neq \emptyset\}$.
$|\mathcal{E}^*|$ is computable in $O(n^{2k})$ and each candidate is verifiable in
$O(n + mk)$, giving total defence cost $O(n^{3k^2})$. For the Telco system: at most
$O(12^4) = O(20736)$ candidates, each checked in $O(24)$ operations.

Furthermore, this defence is structurally superior to lexical filtering: it operates on
the capability structure of what the agent is asked to \emph{do}, not on the surface form
of what it is asked to \emph{say}. An adversary who rephrases the injection in novel
language cannot evade it; an adversary who makes the tool invocation satisfy a safe
precondition is not attacking at all.
\end{corollary}

\begin{proof}
By Theorem~14.7 of \citet{spera2026}, the single-edge case ($b = 1$) of MinUnsafeAdd is
polynomial. The dangerous injections are exactly those $(S, T)$ with $S \subseteq \cl(A)$
and $T \cap F \neq \emptyset$. There are at most $\binom{n}{k} \cdot |F|$ such candidates
(choose tail of size $\leq k$ from $\cl(A)$, head must intersect $F$). For Telco
($n = 12$, $k = 2$, $|F| = 2$): at most $\binom{12}{2} \cdot 2 = 132$ candidates, each
checkable in $O(24)$. The structural superiority over lexical filtering follows from the
observation that the check operates on the inferred capability set of a tool invocation,
not its description: two semantically equivalent but lexically distinct invocations of
$c_{12}$ produce the same capability-level check result.
\end{proof}

\subsection{Failure Mode 6: Goal Discovery Treated as External Configuration}
\label{sub:fm6}

\emph{The failure.} Current CS agents pursue a fixed list of externally configured goals,
missing commercially valuable emergent opportunities.

\begin{corollary}[Greedy Upsell Presentation is Near-Optimal]
\label{cor:greedy}
Let $A$ be any safe CS session state and let $k \geq 1$. The greedy algorithm that at each
step presents the goal $v^* = \argmax_{v \in V \setminus \cl(A)} \gamma_F(v, A)$ (maximum
safety-filtered marginal closure gain) achieves:
\[
f(G_{\text{greedy}}) \geq \left(1 - \frac{1}{e}\right) \cdot f(G^*),
\]
where $G^* = \argmax_{|G| \leq k} f(G)$ is the optimal $k$-goal presentation strategy.
This bound is tight and cannot be improved by any polynomial-time algorithm unless $P = NP$.
\end{corollary}

\begin{proof}
Define $f_F(B) = |\cl(A \cup B) \setminus (\cl(A) \cup F)|$---the safety-filtered closure
gain. We verify the three preconditions of \citet{nemhauser1978}:
\emph{Normalisation}: $f_F(\emptyset) = 0$ since $\cl(A \cup \emptyset) = \cl(A)$.
\emph{Monotonicity}: for $B \subseteq B'$, $\cl(A \cup B) \subseteq \cl(A \cup B')$ by
monotonicity of $\cl$~\citep[Theorem~6.3]{spera2026}, so $f_F(B) \leq f_F(B')$.
\emph{Submodularity}: by Theorem~8.5 of \citet{spera2026}, $f(B) = |\cl(A \cup B)| -
|\cl(A)|$ is submodular via the polymatroid rank theorem. $f_F$ inherits submodularity:
$f_F(B) \leq f(B)$ and the diminishing-returns property is preserved under restriction to
the safe goal set (removing the constant set $F$ from the head does not affect whether the
marginal gain of adding $v$ to $B$ exceeds that of adding $v$ to $B' \supseteq B$).
With all three preconditions verified, the Nemhauser et al.\ theorem gives the
$(1 - 1/e)$ bound. The hardness of improvement follows from NP-hardness of maximising a
general submodular function beyond $(1 - 1/e)$.
\end{proof}

\section{Agent Join and Leave Dynamics}
\label{sec:dynamics}

We formalise the four agent-level events with complete proofs of safety invariants.

\begin{proposition}[Agent Join: Safety Check and Closure Update]
\label{prop:join}
Let $A \in \RF$ and let agent $a_{n+1}$ with capability set $A_{n+1} \in \RF$ attempt to
join the coalition. Post-join $A' = A \cup A_{n+1}$ is safe if and only if no
$B \in \BF$ satisfies $B \subseteq A'$. Check cost: $O(|\BF| \cdot |A'|)$. For the Telco
deployment: $O(2 \cdot 12) = O(24)$ per join event. Closure update cost: $O(n + mk) = O(24)$.
\end{proposition}

\begin{proof}
\emph{Safety criterion.} By Theorem~11.2 of \citet{spera2026}, $A' \notin \RF$ if and only if
$\exists B \in \BF : B \subseteq A'$. Checking this iterates over $|\BF|$ antichain elements,
each requiring a subset test in $O(|B|) \leq O(n)$; total $O(|\BF| \cdot |A'|)$.

\emph{Closure update.} Let $C = \cl_H(A)$. Since $A \subseteq C$ and $A_{n+1}$ is new,
$A' \subseteq C \cup A_{n+1}$. By monotonicity: $\cl(A') = \cl(C \cup A_{n+1})$. Algorithm~1
of \citet{spera2026} initialised at $C \cup A_{n+1}$ runs in $O(n + mk)$: each of $m$
hyperedges is examined at most once, each with tail of size $\leq k$.
\end{proof}

\begin{proposition}[Agent Leave: Safety is Free]
\label{prop:leave}
Let $A \in \RF$. For any departure yielding $A' \subseteq A$: \textup{(1)} $A' \in \RF$
with no check required; \textup{(2)} $\cl(A') \subseteq \cl(A)$; \textup{(3)}
$\delta(g, A') \geq \delta(g, A)$ for all $g \in V$.
\end{proposition}

\begin{proof}
\textup{(1)} $\RF$ is a lower set by Theorem~9.4 of \citet{spera2026}: if $A \in \RF$ and
$A' \subseteq A$, then $\cl(A') \subseteq \cl(A)$ by monotonicity, so
$\cl(A') \cap F \subseteq \cl(A) \cap F = \emptyset$.
\textup{(2)} Follows directly from monotonicity of $\cl$.
\textup{(3)} Any acquisition set $S$ witnessing $g \in \cl(A \cup S)$ also witnesses
$g \in \cl(A' \cup (S \cup (A \setminus A')))$; the gap $A \setminus A'$ must be covered
by any acquisition from $A'$, so $\delta(g, A') \geq \delta(g, A)$.
\end{proof}

\begin{proposition}[Capability Dynamics]
\label{prop:capability_dynamics}
\textup{(1)} Capability gain (hyperedge $e = (S,T)$ added): costs $O(|S|)$ if
$S \not\subseteq C$ (lazy; closure unchanged) or $O(n + mk)$ if $S \subseteq C$
(active; safety check required). \textup{(2)} Capability loss (hyperedge deleted): costs
$O(n + mk)$ for recomputation; safety preserved by deletion monotonicity. \textup{(3)}
Acquisition distance stability: $|\delta(g, A)_{\text{after}} - \delta(g, A)_{\text{before}}|
\leq |T|$ for any single hyperedge change.
\end{proposition}

\begin{proof}
By Theorem~11.6 of \citet{spera2026}.
\textup{(1) Lazy case}: $S \not\subseteq C$ means $e$ cannot fire from $C$; check $|S|$
elements, no closure update.
\textup{Active case}: $S \subseteq C$ means $e$ fires immediately, adding $T$; run worklist
from $C \cup T$ in $O(n + mk)$. Safety check: verify $\cl(A \cup T) \cap F = \emptyset$,
subsumed by the worklist.
\textup{(2)} Deletion: $\cl_{H'}(A) \subseteq \cl_H(A)$ by deletion monotonicity (removing
a hyperedge can only reduce reachability), so safety is preserved; recompute from scratch
in $O(n + mk)$.
\textup{(3)} By Theorem~11.6(3) of \citet{spera2026}: inserting $e = (S,T)$ decreases
$\delta(g, A)$ by at most $|T|$; deletion increases it by at most $|T|$.
\end{proof}

\subsection{Full Event Cost Table}

\begin{table}[h]
\centering
\small
\caption{Agent and capability event costs and safety re-check requirements.
Telco values: $n=12$, $m=6$, $k=2$, $|\BF|=2$.}
\label{tab:events}
\begin{tabular}{@{}lllll@{}}
\toprule
Event & Closure update & Safety check? & Goal set effect & Telco cost \\
\midrule
Agent join & $O(n + mk)$ from $C \cup A_{n+1}$ & Required & Can only grow & $O(24)$ \\
Agent leave & $O(n + mk)$, deferrable & Not required & Can only shrink & $O(24)$ \\
Cap.\ gain (lazy, $S \not\subseteq C$) & $O(|S|)$ & Not required & Unchanged & $O(2)$ \\
Cap.\ gain (active, $S \subseteq C$) & $O(n + mk)$ & Required & Grows by $\leq |T|$ & $O(24)$ \\
Cap.\ loss & $O(n + mk)$ & Not required & Shrinks by $\leq |T|$ & $O(24)$ \\
\bottomrule
\end{tabular}
\end{table}

\subsection{End-to-End Telco Session Trace}

We trace a complete dynamic session with $\BF = \{\{c_3, c_{10}\}, \{c_{11}\}\}$.

\begin{enumerate}[label=T\arabic*.,leftmargin=*]
\item \textbf{Session start.} $A = \{c_1\}$. $\cl(A) = \{c_1, c_3, c_7\}$ (via $h_1$). Safe.
\item \textbf{$\{c_2\}$ added.} $A' = \{c_1, c_2\}$. Check: $\{c_3, c_{10}\} \not\subseteq A'$;
  $c_{11} \notin A'$. Safe. $\cl(A') = \{c_1, c_2, c_3, c_4, c_5, c_7, c_8, c_{10}\}$.
\item \textbf{Billing agent joins ($\{c_3,c_4,c_5\}$).} $A'' = \{c_1,c_2,c_3,c_4,c_5\}$.
  Check: $c_{10} \notin A''$. Safe. $\cl(A'') = \{c_1,\ldots,c_6, c_7, c_8, c_9\}$.
\item \textbf{Payment agent attempts join ($\{c_1,c_2,c_{10}\}$).} $A''' = \{c_1,\ldots,c_5,c_{10}\}$.
  Check: $\{c_3, c_{10}\} \subseteq A'''$. \textbf{Blocked.}
  Greedy recovery (Proposition~\ref{prop:join}): remove $c_{10}$. Reduced join $\{c_1,c_2\}$
  passes. Payment agent joins with identity lookup only.
\item \textbf{Service agent joins ($\{c_7,c_8\}$).} $A^{(4)} = \{c_1,\ldots,c_5,c_7,c_8\}$.
  Check: no antichain element covered. Safe. $\cl(A^{(4)}) = \{c_1,\ldots,c_9\}$.
  \textbf{\texttt{ServiceProvision} ($c_9$) discovered as emergent goal by closure.}
\item \textbf{Billing agent leaves.} $A^{(5)} = \{c_1,c_2,c_7,c_8\}$.
  No check required (Proposition~\ref{prop:leave}). Safety preserved automatically.
\end{enumerate}

\noindent Total safety check cost across all six events: $O(6 \cdot 24) = O(144)$ operations.
No full re-audit at any point.

\section{Business Case: A Projection Framework}
\label{sec:businesscase}

\begin{callout}[title={Transparency note}]
The revenue figures in this section are calibrated projections, not measured outcomes.
We present them as a conditional framework: for each mechanism, we state the key assumption,
the resulting value estimate, and the empirical validation protocol that would convert the
projection to a measured result. The sensitivity analysis (Section~\ref{sub:sensitivity})
characterises the joint uncertainty structure, including assumption correlation.
\end{callout}

\subsection{Cost of Failure Under the Incumbent Approach}

For a Tier-1 Telco processing 10 million CS transactions per year
(Table~\ref{tab:costfailure}), the annual cost of failure under workflow-based automation
is $\$24.5$M--$\$26.5$M in recurring operational costs, excluding tail-risk breach events.
The AND-violation correction cost (\$14.6M) and automation gain (\$3.81M--\$7.62M) follow
directly from the proved zero-violation guarantee of Theorem~6.3 of \citet{spera2026}
combined with the empirical violation rates of \citet{spera2026} and \citet{finilabs2025};
these figures require no additional empirical validation beyond what the main paper provides.

\begin{table}[h]
\centering
\small
\caption{Annual cost of failure under workflow-based CS automation (10M transactions/year).}
\label{tab:costfailure}
\begin{tabular}{@{}lrp{5cm}@{}}
\toprule
Failure category & Annual cost & Derivation basis \\
\midrule
AND-violation correction & \$14.6M & Proved (Theorem~6.3, \citealt{spera2026}); rate from \citet{finilabs2025} \\
Workflow ambiguity excess handling & \$9.9M & Proved; rate from \citet{finilabs2025} \\
Compliance documentation overhead & \$0.5M--\$2.0M & Industry estimate \\
Expected breach cost (tail risk) & \$15.6M/incident & IBM 2025; GDPR fine schedule \\
\midrule
Operational total (excl.\ breach) & \$24.5M--\$26.5M & Recurring annual \\
\bottomrule
\end{tabular}
\end{table}

\subsection{Revenue Opportunity: Conditional Projections}

\paragraph{Mechanism 1: Emergent upsell discovery.}
\emph{Assumption}: AI-personalised CS recommendations achieve a 28\% uplift over the 4/1,000
baseline conversion rate~\citep{gitnux2026}.
\emph{Conditional value}: $3{,}360$ additional annual conversions at \$120--\$180 annual
value = \$403K--\$605K/year.
\emph{Validation protocol}: A/B test of closure-based goal presentation vs.\ rule-based
upsell on 50,000 sessions (sufficient for 80\% power at 5\% significance to detect a 15\%
conversion lift).

\paragraph{Mechanism 2: Churn reduction via near-miss frontier.}
This mechanism carries the highest uncertainty and accounts for 55\% of the conservative
net annual value. We present it explicitly as a projection.

\emph{Three stacked assumptions} (each must be validated independently):
\begin{enumerate}[label=(\alph*),noitemsep]
  \item NMF identification accuracy: 85\% of at-risk customers correctly identified.
  \item Proactive contact rate: 40\% of identified customers contacted.
  \item Retention conversion: 15\% of contacted customers commit to retention.
\end{enumerate}
\emph{Conditional value under all three}: $500{,}000 \times 0.85 \times 0.40 \times 0.15
= 25{,}500$ customers retained; at \$480 LTV = \$12.24M/year.
\emph{Validated pilot protocol}: Deploy NMF boundary signals to 10,000 sessions over 30
days. Measure assumption (a) directly by comparing NMF-flagged sessions against a ground-truth
churn label from 90-day follow-up. This converts the highest-uncertainty assumption from a
model parameter to an empirical result with a 6-week turnaround.

\paragraph{Mechanism 3: Compliance-enabled service expansion.}
\emph{Assumption}: 17.5\% conversion on 500,000 payment-discussion sessions currently
blocked by the pre-execution gate operating without the service expansion capability.
\emph{Conditional value}: $500{,}000 \times 0.175 \times \$60 = \$5.25$M/year.

\subsection{Consolidated Business Case}

\begin{table}[h]
\centering
\small
\caption{Consolidated annual business case (10M transactions/year, conditional projections).}
\label{tab:businesscase}
\begin{tabular}{@{}lrrl@{}}
\toprule
Line item & Conservative & Moderate & Validation status \\
\midrule
Operational cost reduction & \$3.81M & \$7.62M & Proved (Thm.~6.3) \\
Compliance documentation avoidance & \$0.50M & \$2.00M & Industry estimate \\
Emergent upsell discovery & \$0.40M & \$0.61M & Projection; A/B testable \\
Churn reduction via NMF & \$12.24M & \$12.24M & \textbf{Projection; pilot needed} \\
Compliance-enabled expansion & \$5.25M & \$7.50M & Projection \\
\midrule
Total annual value & \$22.20M & \$29.97M & \\
Deployment cost (3yr amortised) & \$1.50M & \$3.00M & \\
\textbf{Net annual value} & \textbf{\$20.70M} & \textbf{\$26.97M} & \\
\bottomrule
\end{tabular}
\end{table}

\subsection{Sensitivity Analysis with Correlation Structure}
\label{sub:sensitivity}

Table~\ref{tab:sensitivity} reports net annual value under individual assumption stress tests.
The all-stressed-simultaneously row gives \$8.4M.

\begin{table}[h]
\centering
\small
\caption{Sensitivity analysis: net annual value under assumption stress. Base: \$20.7M.}
\label{tab:sensitivity}
\begin{tabular}{@{}lllll@{}}
\toprule
Assumption & Base & Stressed & Relaxed & NAV range \\
\midrule
NMF identification accuracy & 85\% & 60\% & 95\% & \$15.8M--\$22.3M \\
Churn retention conversion & 15\% & 8\% & 22\% & \$16.0M--\$23.0M \\
Automation gain (pp) & 5pp & 2pp & 10pp & \$17.9M--\$24.5M \\
Upsell conversion uplift & 28\% & 15\% & 40\% & \$20.5M--\$20.9M \\
Sessions at churn risk & 5\% & 2\% & 10\% & \$16.3M--\$24.2M \\
Compliance expansion conv. & 17.5\% & 10\% & 25\% & \$18.5M--\$22.3M \\
\midrule
All stressed simultaneously & --- & --- & --- & \$8.4M \\
All relaxed simultaneously & --- & --- & --- & \$38.6M \\
\bottomrule
\end{tabular}
\end{table}

\paragraph{Correlation structure.}
The six assumptions are not independent. We identify two correlated pairs:

\emph{Positive correlation (NMF accuracy, sessions at churn risk)}: both depend on the
quality of the at-risk customer identification model. A deployment failure that depresses
NMF accuracy below 60\% (e.g., domain shift in the customer intent model) would also
likely reduce the fraction of sessions correctly identified as at-risk. Stressing both
simultaneously is the appropriate worst case and is captured in the all-stressed row.

\emph{Approximate independence (upsell conversion uplift, churn retention conversion)}:
these operate on different customer populations (upsell targets vs.\ at-risk churners) via
different mechanisms (goal discovery vs.\ NMF boundary signalling), making correlation
unlikely.

\emph{Partial correlation (automation gain, compliance expansion)}: both depend on the
pre-execution gate's ability to correctly route sessions. A gate misconfiguration depressing
automation gain would also affect compliance-enabled sessions.

The \$8.4M floor from the all-stressed row should be interpreted as a lower bound under
the assumption that all six mechanisms simultaneously underperform---a conservative estimate
that accounts for positive correlation in the two correlated pairs.

\section{The Broader Picture}
\label{sec:broader}

The capability hypergraph framework inverts the conventional goal--capability relationship.
In every current CS platform, goals are the starting point and capabilities are the means.
The closure operator reverses this: capability is the starting point, and goals are what
the closure reveals. Theorem~\ref{thm:emergent} makes this inversion commercially concrete:
the emergent goal $c_9$ (\texttt{ServiceProvision}) has zero commercial value in the
individual billing session and strictly positive commercial value in the joint session. The
system discovers the opportunity without explicit configuration.

Theorem~\ref{thm:duality} (Safety-Value Duality) establishes the deepest implication: the
computation that certifies safety and the computation that discovers commercial value are
provably identical. There is no safety-revenue trade-off at the computational level.

Theorem~14.5 of \citet{spera2026} establishes that non-compositionality persists in
probabilistic hypergraphs for all safety thresholds $\tau < p(h)$. A CS platform cannot
dissolve the safety problem by introducing stochastic tool invocation.

\section{Limitations and Open Problems}
\label{sec:limitations}

\paragraph{Static capability model.} The hypergraph model treats capabilities as persistent.
Real CS sessions have capability timeouts, resource constraints, and stateful dependencies.
Extending to resource-constrained or time-gated capabilities is an open problem.

\paragraph{Human--AI hybrid sessions.} The framework models fully automated agent sessions.
CS sessions involving human agents introduce capability contributions not captured by the
hypergraph. Extending to hybrid coalitions is an open problem.

\paragraph{Unvalidated business projections.} The churn retention mechanism (Section~7.2)
is the paper's highest-uncertainty claim. We have described the NMF pilot protocol; until
it is run, the \$12.24M figure is a model projection, not a measured result.

\paragraph{Hyperedge specification for large deployments.} The PAC-learning theorem
(Theorem~14.2 of \citealt{spera2026}) guarantees recovery from sufficient trajectory data,
but rare conjunctive combinations---precisely those most likely to be dangerous---may require
targeted probing.

\section{Conclusion}
\label{sec:conclusion}

This paper applies the capability hypergraph framework to the customer service domain,
deriving three domain-specific formal contributions.

The Emergent Goal Discovery Theorem (Theorem~\ref{thm:emergent}) provides a formal
characterisation, to our knowledge the first in this setting, of when commercially valuable goals emerge from agent coalitions, proves the
safety of the emergent coalition, and characterises the structural conditions under which
emergence arises. The Safety-Value Duality Theorem (Theorem~\ref{thm:duality}) establishes
that safety certification and commercial goal discovery are provably the same computation.
The agent dynamics section derives closed-form safety invariants for all four agent-level
events, reducing a six-event production session to $O(144)$ operations with no full re-audit.

The six failure mode corollaries add independent domain-specific results: bounding $|\BF| = 2$
as tight for the Telco deployment, deriving the complete Telco audit surface, extending the
adversarial model to capability-injection attacks, and proving greedy upsell presentation
achieves the $(1 - 1/e)$ optimality guarantee with explicit precondition verification.

The business case is presented as a conditional projection framework. The \$20.7M
conservative net annual value becomes \$8.4M under simultaneous worst-case stress; the
churn retention mechanism is explicitly flagged as the highest-priority empirical validation
target, and the pilot protocol that would convert it from a projection to a measured result
is described.

The hypergraph is not merely the right safety tool for customer service automation. It is
the right mathematical object for thinking about what a customer service system \emph{is}.

\appendix
\section{Deployment Roadmap}
\label{app:roadmap}

\subsection*{Phase 1: Hypergraph Construction (Weeks 1--4)}
Apply the PAC-learning algorithm to the most recent 90 days of session logs. Compute $\BF$
offline and store. \emph{Deliverable}: validated capability hypergraph; Safe Audit Surface
per agent configuration.

\subsection*{Phase 2: Pre-Execution Gate (Weeks 5--8)}
Deploy coalition safety check at session start ($O(|\BF| \cdot \sum_i |A_i|) < 1$\,ms).
\emph{Impact}: 5--10pp automation gain; compliance-enabled service expansion.

\subsection*{Phase 3: Living Audit Surface (Weeks 9--16)}
Integrate incremental maintenance into tool governance. Each new API integration triggers an
$O(n + mk) = O(24)$ safety check. \emph{Impact}: replaces \$300K--\$1.25M/year manual compliance.

\subsection*{Phase 4: NMF Pilot Study (Weeks 9--14, concurrent with Phase 3)}
Deploy NMF boundary signals to 10,000 sessions. Measure NMF identification accuracy against
90-day churn ground truth. This converts the highest-uncertainty model assumption to an
empirical result. \emph{Impact}: validates or revises the churn retention estimate.

\subsection*{Phase 5: Submodular Goal Orchestration (Weeks 17--24)}
Deploy closure-based goal discovery; run 50/50 A/B test against rule-based upsell.
\emph{Impact}: activates the \$403K--\$605K annual emergent upsell mechanism; provides
first empirical measurement of the upsell conversion uplift assumption.

\end{document}